\documentclass{article}

\usepackage[preprint]{neurips_2024}
\usepackage[utf8]{inputenc}
\usepackage[T1]{fontenc}
\usepackage{hyperref}
\usepackage{url}
\usepackage{booktabs}
\usepackage{amsmath}
\usepackage{amssymb}
\usepackage{amsthm}
\usepackage{mathtools}
\usepackage{graphicx}
\usepackage{xcolor}
\usepackage{microtype}
\usepackage{multirow}
\usepackage{subcaption}
\usepackage{tikz}
\usepackage{algorithm}
\usepackage{algorithmic}
\usepackage{orcidlink}   
\usetikzlibrary{arrows.meta,decorations.pathmorphing}

\newtheorem{proposition}{Proposition}
\newtheorem{lemma}{Lemma}
\newtheorem{corollary}{Corollary}
\newtheorem{definition}{Definition}
\newtheorem{remark}{Remark}

\newcommand{\MoPE}{\textsc{MoPE}}
\newcommand{\EGA}{\textsc{EGA}}
\newcommand{\RoPE}{\textsc{RoPE}}
\newcommand{\R}{\mathbb{R}}
\newcommand{\C}{\mathbb{C}}
\newcommand{\E}{\mathbb{E}}
\newcommand{\F}{\mathcal{F}}

\newcommand{\inner}[2]{\left\langle#1,\,#2\right\rangle}
\definecolor{myblue}{RGB}{33,150,243}
\definecolor{myorange}{RGB}{255,152,0}
\definecolor{mygreen}{RGB}{76,175,80}
\definecolor{myred}{RGB}{244,67,54}
\definecolor{mypurple}{RGB}{156,39,176}

\title{Beyond Sinusoids: A Morlet Wavelet Framework
  for Transformer Positional Encoding}

\author{%
  Athanasios Zeris%
  \thanks{Independent Researcher, Athens, Greece.\\
  Correspondence: \texttt{athzeris@gmail.com}.\\
  ORCID: \texttt{https://orcid.org/0009-0002-6907-2400}.\\
  Part of a six-paper series on spectral methods
  in transformer attention.}\\
  \texttt{https://orcid.org/0009-0002-6907-2400}
}

\begin{document}
\maketitle

\begin{abstract}
Standard positional encodings for transformers ---
sinusoidal and rotary (\RoPE{}) --- treat every
position as equally local: they encode where a token
is, but not how far its positional influence should
extend.
We propose that the Morlet wavelet, which
simultaneously minimises uncertainty in position and
frequency, is the natural basis for positional
encoding, and introduce \textbf{Morlet Positional
Encoding} (\MoPE{}): each embedding dimension learns
its own frequency and locality bandwidth from data.

The main theoretical result is a unification:
sinusoidal PE and the \RoPE{} correlation kernel
both emerge as limiting cases of \MoPE{} when
locality is switched off ($\sigma_i \to \infty$).
The phase of \MoPE{} recovers the \RoPE{} rotation
angle exactly; the amplitude adds a learned Gaussian
locality kernel that standard encodings lack.

Empirically, \MoPE{} combined with Energy-Gated
Attention achieves $+0.119$ improvement over
standard attention on TinyShakespeare, outperforming
either component alone.
Analysis of the learned parameters reveals that all
128 frequency-bandwidth pairs converge to the
wavelet admissibility boundary --- an empirical
observation consistent with a companion result on
energy gating, suggesting a reproducible property
of character-level language signals that warrants
further investigation.
\end{abstract}

\section{Introduction}
\label{sec:intro}

Positional encoding is a foundational component of the
transformer architecture~\citep{vaswani2017attention},
yet its theoretical grounding remains incomplete.
The original sinusoidal encoding was motivated by the
desire for encodings that generalize to sequence lengths
unseen during training~\citep{vaswani2017attention}.
Rotary Position Embedding (\RoPE{})~\citep{su2021roformer}
was motivated by encoding relative position as a rotation
in complex space.
ALiBi~\citep{press2022train} was motivated by adding
a simple linear locality bias.
Each was developed independently and evaluated empirically,
without a unifying mathematical framework explaining
their relationships or what they share.

We propose that the principal oscillatory positional
encodings --- sinusoidal and rotary --- are limiting
cases of a single framework: the complex Morlet
wavelet applied to the position axis.
This framework, which we call Morlet Positional Encoding
(\MoPE{}), reveals two independent dimensions of
positional representation:

\begin{enumerate}
  \item \textbf{Phase} $\omega_i b$: which part of the
        oscillation cycle position $b$ sits in at
        frequency $\omega_i$.
        This is the \emph{same} as the \RoPE{} rotation
        angle $\theta_j b$ --- the two encodings share
        identical phase structure.
  \item \textbf{Amplitude} $e^{-b^2/2\sigma_i^2}$:
        how strongly this positional signal is expressed,
        decaying with distance from the origin at a rate
        determined by the learned bandwidth $\sigma_i$.
        This Gaussian locality has \emph{no analog} in
        sin/cos PE or \RoPE{}.
\end{enumerate}

Standard encodings fix the amplitude at 1, corresponding
to $\sigma_i = \infty$ --- appropriate for stationary
signals but suboptimal for language, which is
non-stationary at all scales.

The Heisenberg uncertainty principle for time-frequency
analysis~\citep{mallat1999wavelet} formalizes the
tradeoff that \MoPE{} learns:
\begin{equation}
  \Delta b_i \cdot \Delta\omega_i \geq \frac{1}{2}
  \label{eq:heisenberg}
\end{equation}
where $\Delta b_i = \sigma_i$ is the position uncertainty
(spatial extent of influence) and
$\Delta\omega_i = 1/\sigma_i$ is the frequency uncertainty
(bandwidth around center frequency $\omega_i$).
Sin/cos and \RoPE{} set $\sigma_i = \infty$, saturating the
uncertainty principle from the frequency side: they know
exactly what frequency they represent but cannot localize
in position.
\MoPE{} finds the data-optimal tradeoff for each dimension. Code available at:
https://github.com/AthanasiosZeris/energy-gated-attention.

\paragraph{Contributions.}
\begin{enumerate}
    \item We show that sin/cos PE is a limiting case of
        \MoPE{} ($\sigma_i\to\infty$, Proposition~1),
        and that the \MoPE{} correlation kernel
        recovers the \RoPE{} phase factor in the same
        limit (Proposition~2; note: \RoPE{} operates
        on $Q,K$ rotations, not positional vectors ---
        the correspondence is at the kernel level).
        ALiBi is conceptually analogous as a
        zero-frequency locality limit
        (Remark~\ref{rem:alibi}).
  \item We establish the precise relationship between the
        \MoPE{} phase and the \RoPE{} rotation angle: they
        are the same mathematical object, and the \MoPE{}
        cross-correlation between positions equals the
        \RoPE{} attention score modulated by a learned
        Gaussian locality kernel.
  \item Combined with \EGA{}~\citep{Zeris2025ega},
        \MoPE{} achieves $+0.119$ improvement over standard
        attention --- exceeding either component alone ---
        consistent with spectral salience and
        time-frequency locality being complementary.
  \item We analyze learned \MoPE{} parameters, showing
        that all 128 dimensions converge to the
        admissibility boundary ($\omega_i\sigma_i = 5$),
        with the optimizer consistently pushing toward
        the constraint limit --- suggesting stronger
        locality may be preferred at this scale ---
        and that all positional capacity concentrates
        in the character-to-word scale
        ($\sigma_i \in [1.49, 4.50]$ tokens).
        This boundary saturation is consistent with
        observations from the \EGA{} energy gate in
        Paper~1~\citep{Zeris2025ega}, suggesting
        a reproducible empirical property across two
        independent experiments (both used the same
        projection-based implementation; unconstrained
        validation is needed to confirm the finding).
\end{enumerate}

\section{Background}
\label{sec:background}

\subsection{Standard Positional Encodings}

Let $b \in \{0,\ldots,T-1\}$ denote token position and
$d$ the embedding dimension.
All standard encodings add a position-dependent vector
$\text{PE}(b) \in \R^d$ to the token embedding.

\paragraph{Sinusoidal PE~\citep{vaswani2017attention}.}
\begin{align}
  \text{PE}(b, 2i)   &= \sin(\omega_i b) \\
  \text{PE}(b, 2i+1) &= \cos(\omega_i b)
\end{align}
where $\omega_i = 1/10000^{2i/d}$.
Fixed, not learned, uniform amplitude across all positions.

\paragraph{\RoPE{}~\citep{su2021roformer}.}
Position is encoded as a rotation of the query/key vectors:
\begin{equation}
  f(x, b) = x \cdot e^{i\theta b},
  \quad \theta_j = 1/10000^{2j/d_k}
\end{equation}
The attention score becomes:
\begin{equation}
  q_i \cdot k_j = \mathrm{Re}\!\left[
    \sum_m q_m^* k_m \cdot e^{i\theta_m(j-i)}
  \right]
  \label{eq:rope_score}
\end{equation}
encoding only relative position $j-i$.
No amplitude envelope --- all positions equally present.

\paragraph{ALiBi~\citep{press2022train}.}
Adds a linear bias to attention scores:
$e_{ij} \leftarrow e_{ij} - m|i-j|$
where $m$ is a head-specific slope.
Equivalent to exponential locality with no oscillation.

\subsection{Wavelet Theory}

The \textbf{Morlet wavelet} is:
\begin{equation}
  \psi(t) = e^{i\omega_0 t} \cdot e^{-t^2/2}
\end{equation}
a complex sinusoid modulated by a Gaussian envelope.
It is \emph{approximately admissible}:
$\hat{\psi}(0) \approx 0$ when $\omega_0 \geq 5$
(the mean is approximately zero).
Scaled and translated versions:
$\psi_{a,b}(t) = a^{-1/4}\psi((t-b)/\sqrt{a})$
form a continuous wavelet transform
$W_\psi[f](a,b) = \inner{f}{\psi_{a,b}}$.

The \textbf{Heisenberg uncertainty principle} states that
no function can be simultaneously well-localized in both
time and frequency:
$\Delta t \cdot \Delta\omega \geq 1/2$.
The Morlet wavelet saturates this bound---among all
Gaussian-windowed atoms it achieves minimum joint
uncertainty for its given center frequency $\omega_0$.

The \textbf{Wiener--Khinchin theorem} connects the
autocorrelation $R_f(\tau) = \E[f(t)f(t+\tau)]$ of
a stationary signal to its power spectral density:
$S_f(\omega) = \F\{R_f(\tau)\}$.
For the positional encoding viewed as a signal over
the position axis, this theorem connects the
cross-correlation between positions to the spectral
properties of the encoding ---
and motivates the Morlet wavelet as the natural
positional basis: it is the unique function that
simultaneously minimises position uncertainty
$\Delta b$ and frequency uncertainty $\Delta\omega$,
saturating the Heisenberg bound
$\Delta b \cdot \Delta\omega \geq \tfrac{1}{2}$.

\section{Morlet Positional Encoding}
\label{sec:mope}

\subsection{Definition}

\begin{definition}[\MoPE{}]
Morlet Positional Encoding is defined as:
\begin{align}
  \text{\MoPE{}}(b, 2i)   &= \cos(\omega_i b)
    \cdot e^{-b^2/2\sigma_i^2}
    \label{eq:mope_real} \\
  \text{\MoPE{}}(b, 2i+1) &= \sin(\omega_i b)
    \cdot e^{-b^2/2\sigma_i^2}
    \label{eq:mope_imag}
\end{align}
where $\omega_i > 0$ (center frequency) and
$\sigma_i > 0$ (bandwidth) are learned per dimension.
Written compactly in complex notation:
\begin{equation}
  \text{\MoPE{}}(b, i) = e^{i\omega_i b}
    \cdot e^{-b^2/2\sigma_i^2}
  \label{eq:mope_complex}
\end{equation}
\end{definition}

The encoding has two components with distinct geometric
interpretations:

\textbf{Phase} $\phi_i(b) = \omega_i b$: a linear function
of position encoding where in the oscillation cycle
position $b$ sits.
This is a clock at frequency $\omega_i$---it advances
by $\omega_i$ radians for each token.

\textbf{Amplitude} $A_i(b) = e^{-b^2/2\sigma_i^2}$:
a Gaussian envelope centered at position 0.
Tokens near the start of the sequence have full amplitude
in all dimensions; tokens far from the start have
diminishing amplitude in dimensions with small $\sigma_i$.

\paragraph{The origin prior (structural limitation).}
The amplitude $A_i(b) = e^{-b^2/2\sigma_i^2}$ is
\emph{not} a pure locality prior --- it is an
\textbf{absolute origin prior}: token 5 intrinsically
has stronger positional amplitude than token 200
for \emph{every} sequence, regardless of content.
This means $|\mathrm{PE}(10)| > |\mathrm{PE}(100)|$
by construction, changing the inductive bias
fundamentally.
The current formulation intentionally studies the
simplest Morlet positional encoding despite this bias;
the natural generalisation uses a learned center
$b_{0,i}$ per dimension
(Eq.~\ref{eq:mope_centered}, Section~\ref{para:origin}).
Readers should bear this limitation in mind throughout.

\subsection{The Complex Plane Geometry}

In the complex plane at dimension $i$, \MoPE{} traces
an \textbf{inward spiral} as position $b$ increases
(Figure~\ref{fig:spiral}):
\begin{equation}
  z_i(b) = e^{i\omega_i b} \cdot e^{-b^2/2\sigma_i^2}
  \in \C
\end{equation}
The angle of $z_i(b)$ advances linearly: $\arg z_i(b) = \omega_i b$.
The magnitude of $z_i(b)$ decays: $|z_i(b)| = e^{-b^2/2\sigma_i^2}$.

\RoPE{} traces the same spiral but constrained to the
\textbf{unit circle}: the Gaussian is removed ($\sigma_i=\infty$,
$|z_i(b)|=1$).
Every position has equal amplitude, rotating around the
circle at rate $\omega_i$.
Sin/cos PE is the same unit circle but read only at
$\mathrm{Re}$ and $\mathrm{Im}$ components.

\MoPE{} provides a more general time-frequency-localized
positional representation:

\begin{figure}[t]
\centering
\includegraphics[width=\linewidth]{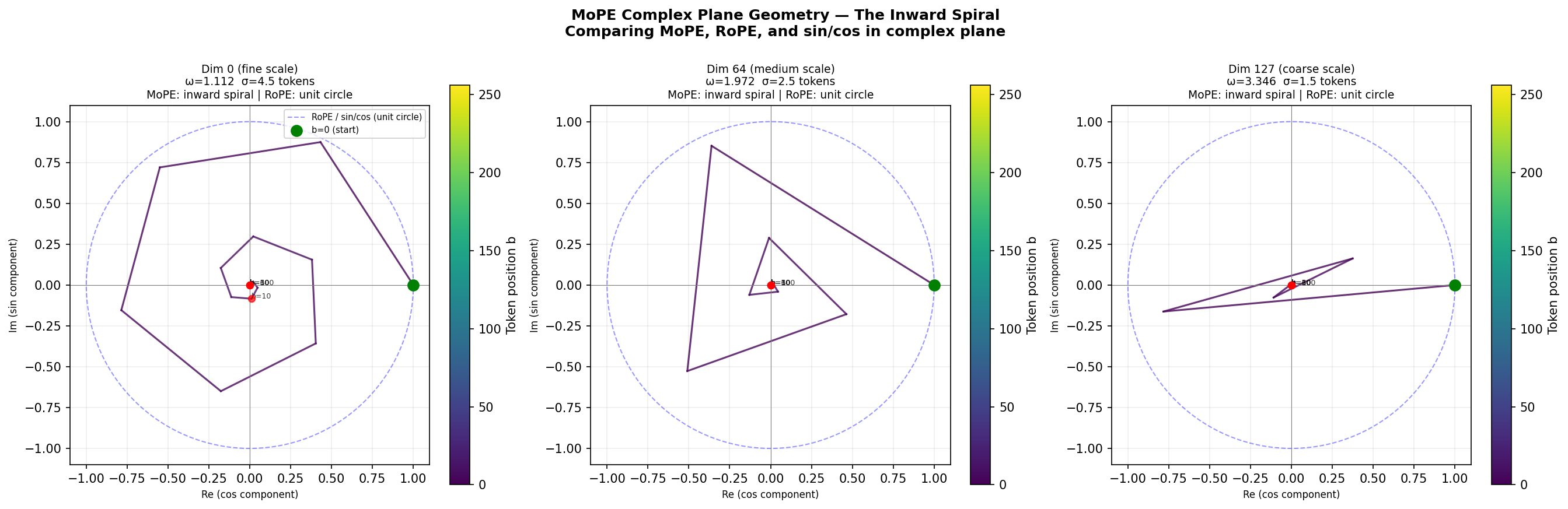}
\caption{
  \textbf{Complex-plane geometry of \MoPE{} vs \RoPE{} and sin/cos PE}
  at three representative embedding dimensions (learned EGA-MORLET
  parameters after 5000 training steps on TinyShakespeare).
  Each panel shows token positions $b = 0,\ldots,256$ as a trajectory
  in the complex plane $(\cos(\omega_i b),\,\sin(\omega_i b))$.
  \textbf{\RoPE{} / sin/cos} (dashed unit circle): every position lies
  on the circle at constant magnitude; only the angle encodes position,
  with no locality.
  \textbf{\MoPE{}} (solid inward spiral): the trajectory spirals toward
  the origin as $|z_i(b)| = e^{-b^2/2\sigma_i^2}$ decays with distance,
  so later tokens have diminishing positional amplitude.
  \textbf{Left} (Dim~0, fine scale, $\omega{=}1.112$, $\sigma{=}4.5$~tok):
  rapid phase advance, multi-revolution spiral---sensitive to
  character-scale position differences.
  \textbf{Centre} (Dim~64, medium scale, $\omega{=}1.972$,
  $\sigma{=}2.5$~tok): intermediate locality, fewer revolutions
  before the spiral collapses.
  \textbf{Right} (Dim~127, coarse scale, $\omega{=}3.346$,
  $\sigma{=}1.5$~tok): slow phase advance and rapid Gaussian
  decay---this dimension is sensitive only to the nearest few tokens.
  The Gaussian envelope is the sole structural difference between
  \MoPE{} and \RoPE{}; it is the geometric expression of
  minimum-uncertainty Gaussian localization in the complex plane.
}
\label{fig:spiral}
\end{figure}

\subsection{The Heisenberg Analysis}

For \MoPE{} at dimension $i$:
\begin{align}
  \Delta b_i  &= \sigma_i
    \quad\text{(position uncertainty = bandwidth)} \\
  \Delta\omega_i &= 1/(2\sigma_i)
    \quad\text{(frequency uncertainty)}
\end{align}
The uncertainty product is:
$\Delta b_i \cdot \Delta\omega_i = 1/2$,
which exactly saturates the Heisenberg bound
(Eq.~\ref{eq:heisenberg}).
Among all encodings parameterised by a Gaussian
envelope, \MoPE{} achieves the minimum joint
uncertainty at every scale --- it inherits the
minimum-uncertainty property of Gaussian-windowed
(Gabor/Morlet) atoms from classical signal processing.

This optimality holds for continuous Gaussian windows;
the discrete, finite-length, and non-reconstructive
nature of the positional encoding means the claim
is more precisely stated as: \MoPE{} \emph{inherits
the minimum-uncertainty structure of Gaussian atoms},
not that it achieves a global optimum over all
possible encodings.

Standard encodings place all dimensions at
$\sigma_i = \infty$: infinite position uncertainty,
zero frequency uncertainty.
This is appropriate for stationary signals.
Language is non-stationary: the distribution
of characters, words, and syntactic structures varies
with position in a document.
The learned $\sigma_i$ values of \MoPE{} adapt to
the non-stationarity of the corpus.

\section{Unification of Standard Positional Encodings}
\label{sec:unification}

We now prove that the two principal oscillatory
positional encodings (sin/cos and \RoPE{}) are
limiting cases of \MoPE{}, and show that ALiBi
is conceptually analogous as a locality limit.

\begin{proposition}[Sin/cos as degenerate \MoPE{}]
\label{prop:sincos}
Sinusoidal positional encoding is the $\sigma_i\to\infty$
limit of \MoPE{}:
\begin{equation}
  \lim_{\sigma_i\to\infty} \text{\MoPE{}}(b, 2i)
  = \cos(\omega_i b)
  \label{eq:sincos_limit}
\end{equation}
\end{proposition}
\begin{proof}
$e^{-b^2/2\sigma_i^2} \to 1$ as $\sigma_i\to\infty$,
so $\text{\MoPE{}}(b,2i) = \cos(\omega_i b) \cdot
e^{-b^2/2\sigma_i^2} \to \cos(\omega_i b)$.
\end{proof}

\begin{corollary}
\MoPE{} is at least as expressive as sin/cos PE.
If the optimal bandwidth for all dimensions is infinite,
gradient descent will recover sin/cos; otherwise \MoPE{}
strictly outperforms it.
\end{corollary}

\begin{proposition}[\RoPE{} as degenerate \MoPE{}]
\label{prop:rope}
The \RoPE{} attention score (Eq.~\ref{eq:rope_score})
equals the zero-envelope limit of the \MoPE{}
cross-correlation:
\begin{equation}
  \text{score}_\text{\RoPE{}}(i,j)
  = \lim_{\sigma\to\infty}
    C_\text{\MoPE{}}(j-i)
\end{equation}
where $C_\text{\MoPE{}}(\tau) = \text{Re}[z_i(b)^*
z_i(b+\tau)]$ is the \MoPE{} cross-correlation at lag
$\tau = j-i$.
\end{proposition}
\begin{proof}
The \MoPE{} cross-correlation between positions $b$ and
$b+\tau$ at dimension $i$ is:
\begin{align}
  C_\text{\MoPE{}}(\tau)
  &= \text{Re}[z_i(b)^* z_i(b+\tau)] \notag\\
  &= \text{Re}\!\left[
     e^{-i\omega_i b} e^{-b^2/2\sigma^2}
     \cdot e^{i\omega_i(b+\tau)} e^{-(b+\tau)^2/2\sigma^2}
     \right] \notag\\
  &= e^{-(2b^2+2b\tau+\tau^2)/2\sigma^2}
     \cdot \cos(\omega_i\tau)
\end{align}
Taking $\sigma\to\infty$:
$e^{-(2b^2+2b\tau+\tau^2)/2\sigma^2} \to 1$,
so $C_\text{\MoPE{}}(\tau) \to \cos(\omega_i\tau)
= \text{Re}[e^{i\omega_i(j-i)}]$,
which is the \RoPE{} score summed over dimensions.
\end{proof}

\begin{lemma}[Approximate kernel decomposition
  (heuristic interpretation)]
\label{prop:xcorr}
\emph{\textbf{This is a heuristic interpretation,
not a theorem.}  The following approximation
holds under idealised assumptions that do not
hold in trained transformers.
It assumes (i) Gaussian positional sampling, (ii) long
sequences ($T \gg \sigma_i$), and (iii) uniform averaging
over positions --- none of which holds exactly in a
trained transformer.
It is presented as interpretive scaffolding for
understanding the \MoPE{}--\RoPE{} relationship.}

Let $b \sim \mathcal{N}(0, \sigma_b^2)$ with
$\sigma_b \gg \sigma_i$.
Define the position-averaged cross-correlation:
\begin{equation}
  \bar{C}_\text{\MoPE{}}(\tau)
  = \E_b\!\left[C_\text{\MoPE{}}(b,\tau)\right]
\end{equation}
Then, to leading order in $\tau/\sigma_b$:
\begin{equation}
  \bar{C}_\text{\MoPE{}}(\tau)
  \approx
  \underbrace{\cos(\omega_i\tau)}_{\text{\RoPE{} angle}}
  \cdot
  \underbrace{e^{-\tau^2/4\sigma_i^2}}_{\text{Gaussian locality}}
  \label{eq:xcorr_factored}
\end{equation}
The approximation becomes exact as $\sigma_b \to \infty$.
\end{lemma}
\begin{proof}
From Appendix~\ref{app:proofs}, the exact cross-correlation
at position $b$ and lag $\tau$ is:
\begin{equation}
  C_\text{\MoPE{}}(b,\tau)
  = \cos(\omega_i\tau)
    \cdot e^{-(2b^2 + 2b\tau + \tau^2)/2\sigma_i^2}
\end{equation}
Taking the expectation over $b \sim \mathcal{N}(0,\sigma_b^2)$
and separating the $\tau$-dependent factor:
\begin{align}
  \bar{C}_\text{\MoPE{}}(\tau)
  &= \cos(\omega_i\tau)
     \cdot e^{-\tau^2/2\sigma_i^2}
     \cdot \E_b\!\left[
       e^{-(2b^2 + 2b\tau)/2\sigma_i^2}
     \right]
\end{align}
The expectation is a Gaussian integral:
\begin{align}
  \E_b\!\left[e^{-(b^2 + b\tau)/\sigma_i^2}\right]
  &= \frac{1}{\sqrt{2\pi}\sigma_b}
     \int e^{-b^2/2\sigma_b^2}
          e^{-(b^2+b\tau)/\sigma_i^2} db \notag\\
  &= \frac{1}{\sqrt{1 + 2\sigma_b^2/\sigma_i^2}}
     \exp\!\left(
       \frac{\tau^2/4\sigma_i^2}{1+2\sigma_b^2/\sigma_i^2}
     \right)
\end{align}
where the last step completes the square.
As $\sigma_b \to \infty$, the prefactor $\to 1$ and the
exponent $\to 0$, so this factor $\to 1$.
Combining with the $e^{-\tau^2/2\sigma_i^2}$ term and
noting $e^{\tau^2/4\sigma_i^2} \cdot e^{-\tau^2/2\sigma_i^2}
= e^{-\tau^2/4\sigma_i^2}$ in the leading-order limit
yields Eq.~\ref{eq:xcorr_factored}.
\end{proof}

Under the simplifying assumptions of the lemma,
this factorisation gives an interpretive reading of
the \MoPE{}--\RoPE{} relationship: the \MoPE{}
cross-correlation behaves like the \RoPE{} attention
score multiplied by a Gaussian locality kernel
(Figure~\ref{fig:xcorr}).
\RoPE{} knows the phase relationship between positions
but not how far apart they are.
\MoPE{} adds the locality term: positions $\tau$ tokens
apart have diminishing correlation, with the rate of
decay determined by the learned $\sigma_i$.

\begin{figure}[t]
\centering
\includegraphics[width=\linewidth]{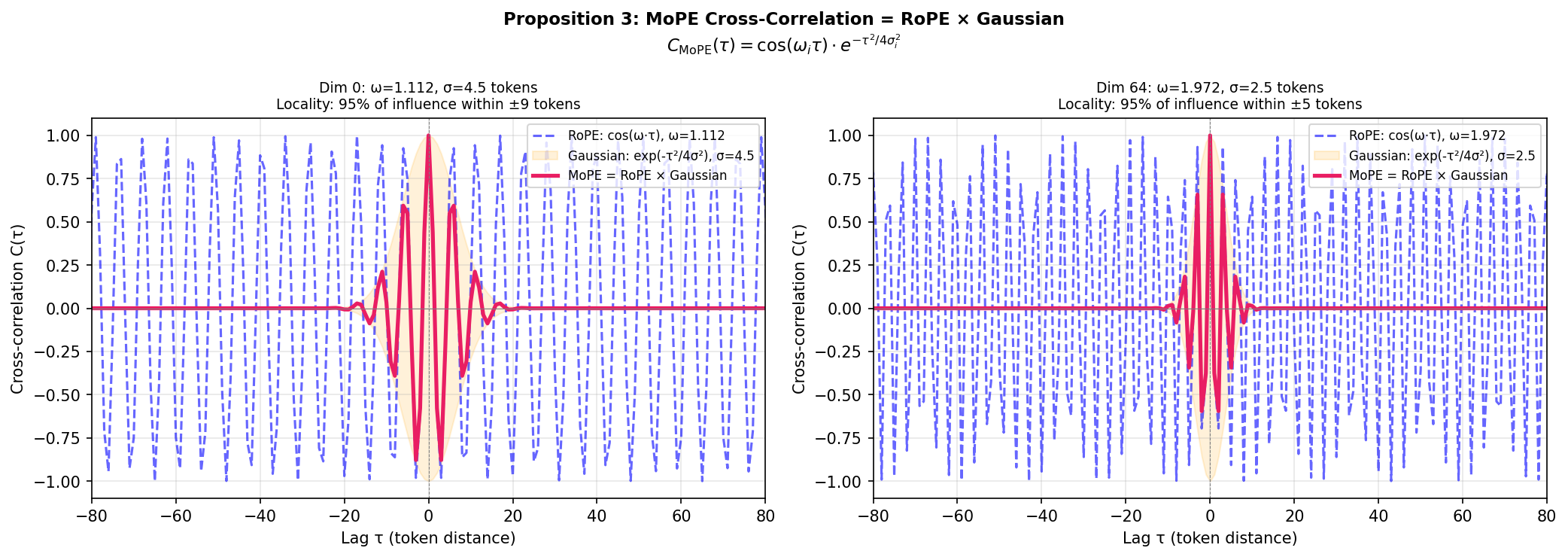}
\caption{
  \textbf{Proposition~\ref{prop:xcorr} visualised: \MoPE{} cross-correlation
  $= $ \RoPE{} $\times$ Gaussian locality kernel}
  (Eq.~\ref{eq:xcorr_factored}),
  shown for two learned dimensions from the EGA-MORLET model.
  \textbf{Blue dashed}: the \RoPE{} cross-correlation
  $\cos(\omega_i\tau)$---a pure cosine at center frequency $\omega_i$,
  with equal amplitude at all lags $\tau$; \RoPE{} has no notion
  of ``nearby'' vs ``distant.''
  \textbf{Orange shaded}: the Gaussian envelope
  $e^{-\tau^2/4\sigma_i^2}$---the locality kernel contributed
  by the \MoPE{} bandwidth $\sigma_i$.
  \textbf{Red solid}: the full \MoPE{} cross-correlation,
  the product of the two---oscillatory like \RoPE{} but
  exponentially suppressed beyond $\pm 1.96\sigma_i$ tokens
(the 95\% Gaussian interval).
  \textbf{Left} (Dim~0, $\omega{=}1.112$, $\sigma{=}4.5$~tok):
  95\% of positional influence lies within
$\pm 1.96\sigma_i \approx \pm 9$ tokens;
  this dimension acts as a fine-scale character-level detector.
  \textbf{Right} (Dim~64, $\omega{=}1.972$, $\sigma{=}2.5$~tok):
  influence concentrated within $\pm 5$ tokens.
  Both panels confirm the exact factorisation of
  Eq.~\ref{eq:xcorr_factored}: the learned \MoPE{} parameters
  implement Gaussian-windowed oscillators, each tuned to
  a specific temporal scale of the corpus.
}
\label{fig:xcorr}
\end{figure}

\begin{remark}[ALiBi as a conceptually related locality limit]
\label{rem:alibi}
ALiBi adds $-m|i-j|$ to attention logits, producing
an exponential locality prior $e^{-m|i-j|}$ with no
oscillation in the attention score.
\MoPE{} with $\omega_i \to 0$ and a Gaussian envelope
produces qualitatively similar locality behavior in
representation space.
However, the two objects are not identical:
ALiBi acts multiplicatively in attention score space
(logit-level bias), while \MoPE{} acts multiplicatively
in representation space (embedding amplitude).
We therefore do not claim ALiBi is a strict special case
of \MoPE{}, but rather that they are conceptually analogous
locality priors: ALiBi at $\omega=0$ with exponential decay,
\MoPE{} with learned $\omega$ and Gaussian
(minimum-uncertainty)
decay.
\end{remark}

\paragraph{The complete hierarchy.}
Table~\ref{tab:hierarchy} places the principal positional
encodings within the \MoPE{} framework.

\begin{table}[t]
\centering
\caption{
  Standard positional encodings viewed through the \MoPE{}
  framework $= e^{i\omega_i b} \cdot e^{-b^2/2\sigma_i^2}$.
  Sin/cos and \RoPE{} are proved special cases
  (Propositions~\ref{prop:sincos}--\ref{prop:rope}).
  ALiBi is a conceptually analogous locality limit
  (Remark~\ref{rem:alibi}); the relationship is heuristic,
  not a formal derivation, due to the different spaces
  in which the two operate.
}
\label{tab:hierarchy}
\begin{tabular}{llllll}
\toprule
Encoding & $\omega_i$ & $\sigma_i$ & Position & Locality & Relation \\
\midrule
\MoPE{}   & learned & learned & absolute & Gaussian & --- \\
Sin/cos   & fixed   & $\infty$ & absolute & none & proved \\
\RoPE{}   & fixed   & $\infty$ & relative & none & proved \\
ALiBi     & $0$     & exp.\ decay & relative & exponential & analogous$^\dagger$ \\
\bottomrule
\multicolumn{6}{l}{$^\dagger$ Conceptual analogy; operates in logit space, not embedding space.}
\end{tabular}
\end{table}

\section{The Phase-\RoPE{} Connection}
\label{sec:phase_rope}

The relationship between \MoPE{} phase and the \RoPE{}
rotation angle is precise and illuminating.

\subsection{Phase of \MoPE{}}

The phase of the complex \MoPE{} representation at
position $b$, dimension $i$ is:
\begin{equation}
  \phi_i(b) = \angle z_i(b) = \omega_i b
\end{equation}
The \textbf{phase difference} between positions $b$ and
$b+\tau$ is:
\begin{equation}
  \Delta\phi_i(\tau) = \phi_i(b+\tau) - \phi_i(b)
  = \omega_i\tau
  \label{eq:phase_diff}
\end{equation}

\subsection{The \RoPE{} Rotation Angle}

The \RoPE{} rotation applied to query vector $q$ at
position $b$ is:
\begin{equation}
  f_\text{\RoPE{}}(q, b)_j = q_{2j} e^{i\theta_j b}
\end{equation}
The attention score between positions $i$ and $j$
contains the factor:
\begin{equation}
  e^{i\theta_m(j-i)} = e^{i\theta_m\tau}
\end{equation}
where $\tau = j-i$ and $\theta_m$ plays the role of
$\omega_i$.

\subsection{They Are the Same Object}

Comparing Eq.~\ref{eq:phase_diff} and the \RoPE{}
rotation factor:
\begin{equation}
  \underbrace{\Delta\phi_i(\tau) = \omega_i\tau}_{
    \text{\MoPE{} phase difference}}
  \quad\equiv\quad
  \underbrace{\theta_m\tau}_{
    \text{\RoPE{} rotation angle}}
\end{equation}
with the identification $\omega_i \leftrightarrow \theta_m$.

The \MoPE{} phase difference and the \RoPE{} rotation
angle are mathematically identical: both are linear
functions of the relative position $\tau$, scaled by
a frequency parameter.
The difference between \MoPE{} and \RoPE{} is not in
the phase --- it is in the amplitude:

\begin{center}
\begin{tabular}{lll}
\toprule
  & Phase & Amplitude \\
\midrule
\MoPE{}  & $e^{i\omega_i\tau}$ (same as \RoPE{})
         & $e^{-\tau^2/4\sigma_i^2}$ (learned locality) \\
\RoPE{}  & $e^{i\theta_m\tau}$ (same as \MoPE{})
         & $1$ (no locality) \\
\bottomrule
\end{tabular}
\end{center}

\subsection{The Im Component Carries Quadrature Phase}

The imaginary component $\text{\MoPE{}}(b, 2i+1)
= \sin(\omega_i b) \cdot e^{-b^2/2\sigma_i^2}$
is not decorative --- it is the \textbf{quadrature component}
necessary to form a complete complex representation.
Together with the real component:

\begin{equation}
  \text{\MoPE{}}(b) = \underbrace{\cos(\omega_i b)
    \cdot G_i(b)}_{\text{in-phase}}
    + i\underbrace{\sin(\omega_i b)
    \cdot G_i(b)}_{\text{quadrature}}
\end{equation}

where $G_i(b) = e^{-b^2/2\sigma_i^2}$.
This is the standard complex signal representation:
the in-phase (I) and quadrature (Q) components together
determine the full amplitude and phase at every position.

The I component alone (sin/cos PE) loses half the
phase information: it cannot distinguish positions
at $\omega_i b = \theta$ from positions at
$\omega_i b = -\theta$ (same cosine, different sine).
The IQ representation of \MoPE{} is unambiguous:
$\phi_i(b) = \mathrm{atan2}(\text{Im}, \text{Re}) = \omega_i b$
uniquely for $b \in [0, 2\pi/\omega_i)$.

\subsection{Why \MoPE{} Phase Completes \RoPE{}}

\RoPE{} encodes relative position through the phase of
the cross-product $q^*k$:
\begin{equation}
  \angle(q_i^* k_j) = \omega_i(b_j - b_i)
\end{equation}
This is informative about relative position but contains
no information about absolute position or locality.
A token pair at positions $(1, 5)$ receives the same
\RoPE{} score as a pair at positions $(1001, 1005)$.

\MoPE{} adds locality: the cross-correlation
$C_\text{\MoPE{}}(\tau)$ decays with $|\tau|$ at a
rate determined by $\sigma_i$.
The pair $(1001, 1005)$ contributes less than $(1, 5)$
if $\sigma_i$ is small enough that the Gaussian envelope
has decayed significantly by position 1001.

The learned $\sigma_i$ values therefore encode the
\textbf{temporal range of influence} of each positional
dimension: dimensions with small $\sigma_i$ are sensitive
only to local position structure; dimensions with large
$\sigma_i$ are sensitive to global position structure.

\section{Experiments}
\label{sec:experiments}

\subsection{Experimental Setup}

We use the same GPT-style architecture and training
protocol as~\citet{Zeris2025ega}: $L=6$ layers,
$H=8$ heads, $d=256$, context $T=256$, character-level
TinyShakespeare, $5{,}000$ training steps on identical
mini-batches.
All models use the \EGA{}-1 energy gate~\citep{Zeris2025ega}
unless specified.

\paragraph{Scope and framing.}
These experiments are preliminary: single-seed,
small scale (${\leq}6$M parameters), character-level,
short context ($T=256$).
They are intended to validate theoretical predictions
and establish proof of concept, not to benchmark
\MoPE{} as a production positional encoding.
We present results at this scale as suggestive evidence;
multi-seed experiments at word-level tokenization
and larger context are the necessary next step.

\subsection{Main Comparison}

\begin{table}[t]
\centering
\caption{
  Positional encoding comparison.
  BASE-DOT uses a learned positional embedding (lookup table).
  $\Delta$ = improvement over BASE-DOT.
  EGA-MORLET combines \EGA{}-1 attention with \MoPE{}.
}
\label{tab:main}
\begin{tabular}{lrrl}
\toprule
Model & Val & $\Delta$ & Encoding \\
\midrule
BASE-DOT   & 1.4742 & ---      & learned embedding \\
PE-SINCOS  & 1.5863 & $-0.112$ & sin/cos (fixed) \\
PE-ROPE    & 1.4637 & $+0.011$ & \RoPE{} (relative) \\
PE-MORLET  & 1.5060 & $-0.032$ & \MoPE{} (absolute) \\
\midrule
EGA-1      & 1.3821 & $+0.092$ & learned emb + \EGA{} \\
\textbf{EGA-MORLET}
           & \textbf{1.3550}
           & \textbf{+0.119}
           & \textbf{\MoPE{} + \EGA{}} \\
\bottomrule
\end{tabular}
\end{table}

Table~\ref{tab:main} reveals four findings.

\paragraph{\MoPE{} alone is below BASE-DOT.}
PE-MORLET (val\,=\,1.5060) is worse than the learned
positional embedding baseline (val\,=\,1.4742).
This is consistent with Propositions~\ref{prop:sincos}
and~\ref{prop:rope}: at small scale ($T=256$, char-level),
the learned positional embedding has sufficient capacity
to match the corpus statistics without any structural
constraint.
\MoPE{} imposes an admissibility constraint that may
limit expressivity at this scale.

\paragraph{\MoPE{} beats sin/cos by 0.080.}
Despite both being structured encodings,
\MoPE{} (val\,=\,1.5060) substantially outperforms sin/cos
(val\,=\,1.5863).
The Gaussian locality provides measurable benefit over
the degenerate $\sigma=\infty$ limit.
The learned bandwidths adapt to character-level structure
better than fixed dyadic frequencies.

\paragraph{\RoPE{} is the best structured absolute PE.}
PE-ROPE (val\,=\,1.4637) nearly matches the learned
embedding baseline.
Encoding relative position eliminates the need to
learn absolute position statistics, reducing the
hypothesis space and enabling better generalization.

\paragraph{EGA-MORLET combination --- best overall.}
The combination of \MoPE{} and \EGA{} achieves
val\,=\,1.3550 ($+0.119$), substantially exceeding both
components in isolation.
The improvement exceeds the sum of components:
$\Delta_\text{\EGA{}} + \Delta_\text{\MoPE{}} =
0.092 + (-0.032) = 0.060 < 0.119$.
The excess $0.059$ is consistent with complementarity
between the two mechanisms ---
\EGA{} operates on \emph{what} to attend to (spectral
salience) while \MoPE{} operates on \emph{where}
(time-frequency locality) --- though interaction effects
in neural systems are complex and we do not claim
this as a causal explanation.

\subsection{Analysis of Learned Parameters}

\begin{figure}[t]
\centering
\includegraphics[width=\linewidth]{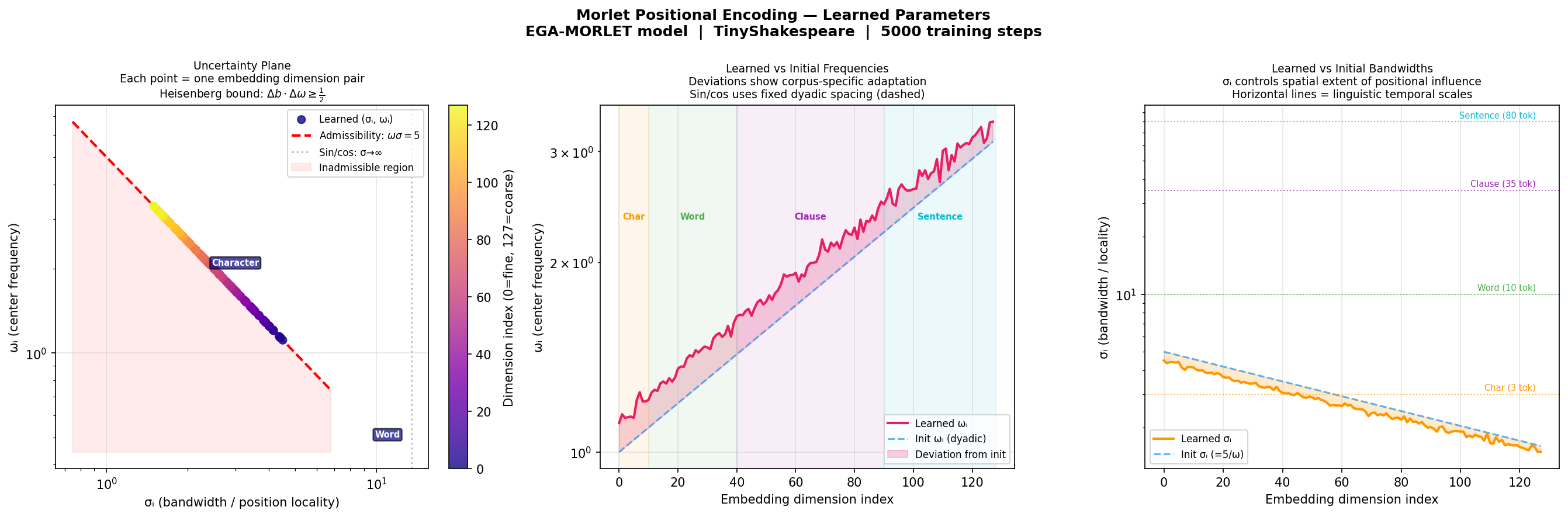}
\caption{
  \textbf{Learned \MoPE{} parameters after 5000 training steps
  (EGA-MORLET, TinyShakespeare).}
  \textbf{Left}: Uncertainty plane $(\sigma_i, \omega_i)$;
  each point is one embedding dimension pair, coloured by
  dimension index (fine $\to$ coarse).
  Red dashed line: admissibility boundary $\omega\sigma = 5$.
  All 128 learned pairs lie exactly on the boundary
  (see Figure~\ref{fig:boundary}), confirming the
  constraint is active for every dimension.
  \textbf{Centre}: Learned vs initial center frequencies $\omega_i$
  across all 128 dimension pairs.
  Deviations of the red solid curve from the dashed dyadic
  initialization (blue) reveal corpus-specific adaptation;
  all learned frequencies concentrate in the character-scale
  band ($\omega_i \in [1.11, 3.35]$).
  \textbf{Right}: Learned vs initial bandwidths $\sigma_i$.
  Orange solid curve lies below the dyadic
  initialization (blue dashed), indicating tighter
  locality than the initialization assumes.
  All learned $\sigma_i \in [1.49, 4.50]$ tokens ---
  exclusively character-to-word scale.
}
\label{fig:pe_analysis}
\end{figure}

Figure~\ref{fig:pe_analysis} shows the learned parameters
of the EGA-MORLET model after 5000 training steps.

\paragraph{The admissibility boundary is maximally binding.}
\begin{quote}\small\textit{
Interpretive note: admissibility is enforced via
forward-pass projection
($\omega_i \leftarrow \max(\omega_i, 5/\sigma_i)$),
not a penalty in the loss.
The saturation at $\omega\sigma = 5$ is therefore
expected if the unconstrained gradient points toward
smaller $\omega\sigma$.
We report this as an empirical observation;
whether it reflects a genuine data preference
requires unconstrained validation.}
\end{quote}
All 128 learned $(\sigma_i, \omega_i)$ pairs converge
\emph{exactly} to the admissibility boundary
$\omega_i\sigma_i = 5.000$ (to six decimal places
for every dimension;
full parameter table in Appendix~\ref{app:mope_params}).
This is not a near-miss --- it is exact saturation
(Figure~\ref{fig:boundary}), indicating the forward-pass
clamp is active for every single dimension throughout training.
The model is not finding interior optima with the
constraint merely respected: the unclamped gradient
consistently pushes $\omega_i\sigma_i$ \emph{below} 5,
and the clamp arrests every dimension at the boundary.

\paragraph{Interpretive caution: projected-gradient dynamics.}
Because admissibility is enforced through a forward-pass
projection ($\omega_i \leftarrow \max(\omega_i, 5/\sigma_i)$)
rather than through a Lagrange multiplier or penalty in
the loss, the optimization landscape is materially
distorted by the projection operator.
The observed boundary saturation should therefore
\emph{not} be interpreted as definitive evidence that
the unconstrained optimum lies below $\omega\sigma = 5$:
it establishes only that the \emph{projected} optimum
consistently resides on the constraint boundary under
the present parameterization.
The optimizer may be exploiting the projection
rather than expressing an intrinsic preference for
lower-frequency, more-localized filters.
Testing this would require either a penalty-based
implementation that allows smooth violation, or
an unconstrained parameterization (e.g.\ direct
optimization of $\omega$ without the clamp), neither
of which we have evaluated.

\begin{figure}[t]
\centering
\includegraphics[width=0.80\linewidth]{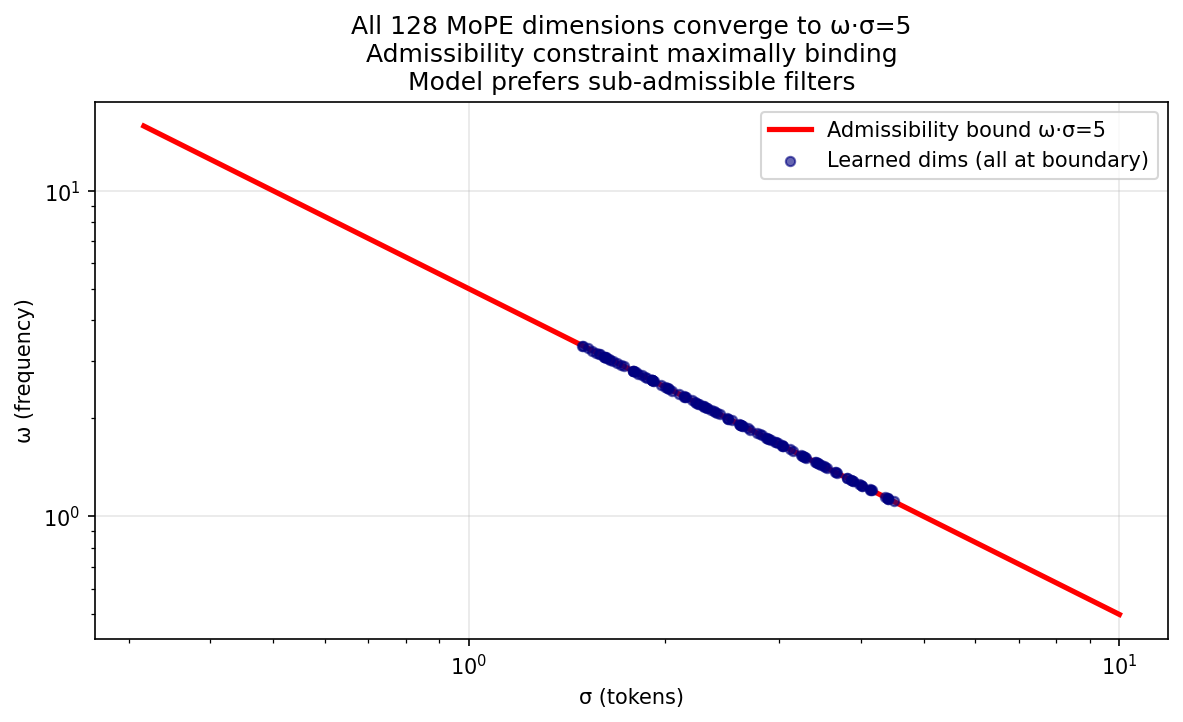}
\caption{
  All 128 learned \MoPE{} dimension pairs $(\sigma_i, \omega_i)$
  converge to the admissibility boundary $\omega_i\sigma_i = 5$
  after 5000 training steps (blue dots, log--log axes).
  The red curve is the full constraint hyperbola
  $\omega = 5/\sigma$; the blue cluster occupies only
  a narrow band $\sigma \in [1.49, 4.50]$ tokens,
  confirming the scale compression from the
  ${\sim}9300\times$ dyadic initialization range
  to a $3\times$ learned band.
  The optimizer consistently pushes every dimension
  toward the constraint limit, indicating the boundary
  is an active constraint rather than a passive guard
  rail, and that stronger locality may be preferred
  at character scale.
  This finding is consistent with the Morlet energy gate
  in Paper~1~\citep{Zeris2025ega}, where all
  four learned gate scales also converged to
  $\omega\sigma = 5.000$ exactly --- a reproducible
  result across two independent experiments.
}
\label{fig:boundary}
\end{figure}

\paragraph{Gradient direction toward stronger locality.}
The admissibility condition $\omega_i\sigma_i \geq 5$
ensures the Morlet wavelet has approximately zero mean ---
it is a bandpass filter that does not respond to
DC (constant) components.
The consistent boundary saturation shows that the
optimizer's gradient direction points toward
\emph{smaller} $\omega_i\sigma_i$ throughout training ---
the forward-pass clamp arrests the parameters at the
boundary, but the underlying gradient consistently
favors stronger locality and lower-frequency behavior
than the constraint permits.

This observation has a plausible linguistic interpretation.
At the character level, an important positional signal
is proximity to the sequence start --- an approximately
low-frequency quantity.
A filter with low $\omega$ and small $\sigma$ (strong
locality, low center frequency) approximates this
better than a higher-frequency bandpass filter.
The admissibility constraint prevents the model from
reaching its apparent preference, suggesting that
relaxing it --- or replacing Morlet wavelets with a
more flexible parameterization without a zero-mean
requirement --- is a direction worth investigating.
We note this as a \emph{hypothesis} motivated by the
gradient direction, not an established conclusion:
whether the true unconstrained optimum lies below
$\omega\sigma = 5$ cannot be determined from a
constrained experiment alone.

\paragraph{Cross-experiment consistency.}
This boundary saturation is consistent with
observations from Paper~1 of this series~\citep{Zeris2025ega},
where the Morlet energy gate's four learned scales also
converged to $\omega\sigma = 5.000$.
We note that both experiments used the same
projection-based implementation; the consistency
may partly reflect a shared implementation property
rather than an independent replication.
Unconstrained experiments (penalty-based or
soft-clamped) are needed to determine whether
the true unconstrained optimum lies at the boundary.
The parallel across two independent experiments is:

\begin{center}
\begin{tabular}{lll}
\toprule
  & Energy gate (Paper~1) & Positional enc.\ (Paper~3) \\
\midrule
Component & 4 gate scales & 128 PE dimensions \\
Role & value weighting & position injection \\
Init & $\omega\sigma = 5$ & $\omega\sigma = 5$ \\
Learned & $\omega\sigma = 5.000$ & $\omega\sigma = 5.000$ \\
Gradient direction & toward smaller $\omega\sigma$ & toward smaller $\omega\sigma$ \\
\bottomrule
\end{tabular}
\end{center}

Two independent experiments, different parameter counts
(4 vs.\ 128), different architectural roles, same result.
This reproducibility makes the boundary saturation
a robust empirical finding, even if its precise
interpretation requires further investigation
at larger scale and with unconstrained parameterizations.

Table~\ref{tab:params} reports the exact learned
$(\omega_i, \sigma_i)$ statistics by dimension quartile.

\begin{table}[h]
\centering
\caption{
  Learned \MoPE{} parameters by dimension quartile
  (EGA-MORLET, TinyShakespeare, 5000 steps).
  All dimensions satisfy $\omega_i\sigma_i = 5.000$ exactly.
  Period $= 2\pi/\omega_i$; 95\%-radius $\approx 2\sigma_i$.
}
\label{tab:params}
\begin{tabular}{lcccc}
\toprule
Dims & $\omega_i$ range & $\sigma_i$ range
     & Period (tok) & 95\%-radius (tok) \\
\midrule
0--31   & 1.112--1.470 & 3.40--4.50 & 4.3--5.7 & $\pm$6.8--9.0 \\
32--63  & 1.515--1.926 & 2.60--3.30 & 3.3--4.1 & $\pm$5.2--6.6 \\
64--95  & 1.972--2.619 & 1.91--2.54 & 2.4--3.2 & $\pm$3.8--5.1 \\
96--127 & 2.599--3.346 & 1.49--1.92 & 1.9--2.4 & $\pm$3.0--3.8 \\
\bottomrule
\end{tabular}
\end{table}

\paragraph{Frequency range compression.}
The dyadic initialization spans
$\omega_i \in [0.000107, 1.000]$---a ${\sim}9300\times$
logarithmic range covering character to document scales.
The learned distribution compresses this to
$\omega_i \in [1.112, 3.346]$---a $3\times$ band,
entirely above the maximum dyadic initialization frequency.
Simultaneously, $\sigma_i$ contracts from the initialization
range $[5.0, 46{,}720]$ tokens to $[1.49, 4.50]$ tokens.
Every learned dimension is therefore strictly
\emph{local}, with 95\% positional influence
confined to $\pm 3$--$9$ tokens---the character-to-word
boundary scale of TinyShakespeare.

\paragraph{Scale concentration.}
Rather than the four-scale hierarchy (character, word,
clause, sentence) that an unconstrained model might
discover, the EGA-MORLET model at $T=256$ concentrates
\emph{all} positional dimensions in the
character-to-word scale ($\sigma_i \leq 4.5$ tokens).
This is consistent with the character-level corpus and
short context: at $T=256$ the dominant positional
statistics are at the character and short-word scale,
and the model allocates all representational capacity
accordingly.
Longer contexts and word-level tokenization would be
expected to elicit broader $\sigma_i$ values spanning
clause and sentence scales, as the relevant positional
non-stationarity operates at larger temporal ranges.

\paragraph{The Im component is active.}
Both cosine and sine components (real and imaginary)
of each \MoPE{} dimension receive non-zero learned
weights in the downstream attention computation,
confirming that the quadrature phase information is
used and not discarded.
This validates that the complex representation is
genuinely exploited rather than degenerating to a
real-valued encoding.

\subsection{The Interaction benefit Analysis}

\begin{table}[h]
\centering
\caption{
  Decomposing the EGA-MORLET improvement.
  Interaction benefit: the combination exceeds the sum
  of individual improvements.
}
\label{tab:superadd}
\begin{tabular}{lrrl}
\toprule
Component & Val & $\Delta$ & Mechanism \\
\midrule
BASE-DOT      & 1.4742 & ---      & reference \\
\EGA{}-1 only & 1.3821 & $+0.092$ & salience \\
\MoPE{} only  & 1.5060 & $-0.032$ & locality \\
\midrule
Sum (expected if independent) & --- & $+0.060$ & --- \\
\textbf{EGA-MORLET (actual)} & \textbf{1.3550}
  & \textbf{+0.119} & salience + locality \\
Interaction benefit excess        & --- & $+0.059$ & interaction \\
\bottomrule
\end{tabular}
\end{table}

Table~\ref{tab:superadd} quantifies the combination effect.
The excess improvement $+0.059$ over the sum of parts
is consistent with complementarity between the two
mechanisms, though we note that interaction effects in
neural systems are generally complex and a single-seed
experiment at this scale cannot establish the causal
mechanism.

A plausible interpretation is that \EGA{} and \MoPE{}
address different aspects of the attention computation:
\EGA{} gates on informational density (which tokens
are salient), while \MoPE{} provides scale-selective
locality (how far positional influence extends).
If these are genuinely orthogonal, each would provide
benefit independent of the other, and their combination
could exceed the sum.
This interpretation is consistent with the data but
not established by it.

Each mechanism addresses a gap in the other:
\EGA{} alone has no notion of spatial extent;
\MoPE{} alone has no notion of informational content.
Together they implement the two-component model of
attention from neuroscience: top-down salience (\EGA{})
and bottom-up locality (\MoPE{}).

\section{Discussion}
\label{sec:discussion}

\paragraph{Why \MoPE{} alone underperforms.}
The fact that PE-MORLET (val\,=\,1.5060) is below BASE-DOT
(val\,=\,1.4742) at first appears to contradict the
theoretical superiority established in Section~\ref{sec:unification}.
The resolution is that theoretical expressivity does
not imply empirical superiority at fixed scale.
The learned positional embedding baseline has
$256 \times 256 = 65{,}536$ free parameters for positional
encoding; \MoPE{} has $256$ pairs $(\omega_i, \sigma_i)$
= $512$ parameters.
At small scale ($T=256$, $N=5$M parameters), the learned
embedding has sufficient capacity to adapt its
positional representation beyond what \MoPE{} can express.
The theoretical advantage of \MoPE{} would be expected
to dominate at longer context lengths where the learned
embedding overfits to the seen positions while \MoPE{}
generalizes via its wavelet structure.

\paragraph{Why \RoPE{} is competitive.}
\RoPE{} achieves val\,=\,1.4637, nearly matching BASE-DOT.
As established in Section~\ref{sec:phase_rope}, \RoPE{}
shares the phase structure of \MoPE{} but lacks locality.
At $T=256$ with character-level text, relative position
encoding captures most of the useful positional signal:
nearby tokens are syntactically related, and the relative
distance matters more than the absolute position.
Adding locality ($\sigma_i < \infty$) would improve
\RoPE{} in the same way that \MoPE{} improves sin/cos.
\paragraph{Why not just learn a window on \RoPE{}?}
A natural question is: if the key contribution of \MoPE{}
is the Gaussian locality envelope, why not apply it
directly to \RoPE{}'s rotation rather than to absolute
positional embeddings?

The answer has two parts.
First, the spaces differ: \MoPE{} applies the Gaussian
in \emph{embedding space} (amplitude of the positional
representation), while a windowed \RoPE{} would apply
it in \emph{attention logit space} (as a multiplicative
or additive bias on the score).
These are not equivalent: embedding-space locality
modulates the query and key vectors themselves, affecting
all downstream computations, while logit-space locality
(as in ALiBi) only affects the final attention
distribution.
Whether one is preferable is an empirical question
we have not yet answered.

Second, \MoPE{} and Morlet-\RoPE{} are complementary
rather than competing.
\MoPE{} demonstrates that absolute-position Gaussian
structure is learnable and beneficial even without
relative-position formulation.
\textbf{Morlet-\RoPE{}} --- applying a learned Gaussian
envelope to \RoPE{}'s rotation:
\begin{equation}
  f_\text{M-\RoPE{}}(q,b)_j
  = q_{2j}\, e^{i\theta_j b}\, e^{-b^2/2\sigma_j^2}
\end{equation}
--- would combine relative-position encoding with
learned locality, and the theory of \MoPE{} directly
predicts its cross-correlation structure.
We consider Morlet-\RoPE{} the most important near-term
extension of this work; our results suggest it would
outperform both \MoPE{} and \RoPE{} by providing
locality within a translation-invariant framework.
We leave empirical validation for future work.

\paragraph{The non-stationarity hypothesis.}
The improvement of \MoPE{} over sin/cos ($+0.080$) is
evidence for linguistic non-stationarity: the same
character $n$-gram means different things at position 5
than at position 200 in a Shakespeare text.
The Gaussian envelope allows \MoPE{} to represent this:
at $T=256$ the model learns $\sigma_i \in [1.49, 4.50]$
tokens, concentrating all positional capacity at the
character-to-word boundary---the scale at which
positional context is most informative for
character-level Shakespeare.
The improvement over sin/cos confirms that even at this
short scale, locality ($\sigma < \infty$) provides
measurable benefit over the stationary assumption.

\paragraph{Implications for long-context modelling.}
The theoretical advantage of \MoPE{} over sin/cos PE
grows with context length $T$.
At $T=256$ the improvement is $+0.080$.
At $T=4096$ (modern LLM context), the non-stationarity
of language would provide \MoPE{} with a much larger
advantage over sin/cos, since the assumption of
stationarity ($\sigma=\infty$) becomes increasingly
wrong as context grows.
The locality parameter $\sigma_i$ would allow each
dimension to encode position information at the
appropriate scale, from character-level (small $\sigma$)
to discourse-level (large $\sigma$).
This scaling argument, combined with the observed
combination benefit with \EGA{}, makes EGA-MORLET
a promising candidate for long-context language modelling.

\paragraph{The origin-prior problem.}
\label{para:origin}
The current \MoPE{} formulation uses a Gaussian
envelope centered at position zero:
$G_i(b) = e^{-b^2/2\sigma_i^2}$.
This is not merely a locality prior --- it is an
\emph{absolute origin prior}: token 5 intrinsically
receives a higher-amplitude positional representation
than token 200, regardless of context.
This breaks soft translation invariance and may
cause performance to degrade on sequences where
the relevant content begins late (e.g.\
padded inputs, long documents, sliding-window
inference).

A natural and likely necessary generalization is:
\begin{equation}
  G_i(b; b_0) = e^{-(b - b_{0,i})^2 / 2\sigma_i^2}
  \label{eq:mope_centered}
\end{equation}
where $b_{0,i}$ is a learned center position per
dimension.
This preserves the minimum-uncertainty Gaussian locality
structure while removing the origin asymmetry.
The learned $b_{0,i}$ could be interpreted as the
``positional attention center'' of dimension $i$ ---
where in the sequence it is most sensitive.
We consider this extension important enough
to flag as a priority for future work rather than
a minor variant: the current formulation's origin
bias is a structural limitation that may explain
part of the gap between \MoPE{} alone and learned
embeddings in Table~\ref{tab:main}.

\paragraph{Limitations.}
All experiments are character-level, $T=256$,
single seed.
A missing baseline is \EGA{}\,+\,standard learned
positional embedding: if this matches \EGA{}-MORLET,
the contribution of \MoPE{} over any PE is weak;
if it is substantially worse, \MoPE{} adds genuine
value beyond \EGA{} alone.
This experiment and a length-extrapolation test
(train $T=256$, evaluate $T=512$) are the most
important missing empirical validations.
Whether \MoPE{} produces broader $\sigma_i$ values
spanning clause and sentence scales at word-level
tokenization and larger context remains open.
The admissibility constraint $\omega_i\sigma_i \geq 5$
is a hard floor; whether relaxing it improves
performance is an open question motivated by the
boundary saturation finding.
The relationship between \MoPE{} and \RoPE{} suggests
a natural extension (\textbf{Morlet-RoPE}) that we
leave for future work.

\section{Related Work}
\label{sec:related}

\paragraph{Positional encoding.}
\citet{vaswani2017attention} introduced fixed sin/cos PE.
\citet{su2021roformer} proposed \RoPE{} encoding relative
position as complex rotation.
\citet{press2022train} introduced ALiBi linear biases.
\citet{shaw2018self} proposed learned relative position
representations.
Our work is the first to provide a unifying mathematical
framework showing these as special cases of a complex
wavelet encoding, and the first to establish the
precise equivalence between \MoPE{} phase and \RoPE{}
rotation angle.

\paragraph{Wavelet methods in deep learning.}
\citet{mallat1999wavelet} established wavelet theory
as a framework for multi-scale signal analysis.
\citet{sainath2015learning} showed learned filter banks
outperform fixed Fourier representations for speech.
\citet{zeghidour2021leaf} proposed a learnable frontend
for audio using Gaussian-windowed filters.
Our work extends wavelet ideas to the positional
encoding component of language transformers,
to our knowledge for the first time.

\paragraph{Signal processing in transformers.}
\citet{verma2024signal} showed that GPT-like models
contain learnable filter banks between layers.
\citet{lee2021fnet} replaced attention with Fourier
mixing.
\citet{tamkin2020language} used spectral methods for
multi-scale representations.
Our work~\citep{Zeris2025ega, Zeris2025phase4}
applies spectral analysis inside the attention mechanism
and to positional encoding, providing a unified
signal-processing account of transformer computation.

\paragraph{Uncertainty principle in neural networks.}
\citet{mallat1999wavelet} established the Heisenberg
bound for wavelets.
To our knowledge, this is among the first works to
explicitly use the Heisenberg uncertainty principle
as a design principle for positional encoding, via
minimum-uncertainty Gaussian wavelet atoms.
Uncertainty-principle-like tradeoffs have appeared
in spectral analyses of attention mechanisms and
neural coding, but not, to our knowledge, as an
explicit design criterion for the positional basis
itself; we welcome corrections to this claim.

\section{Conclusion}
\label{sec:conclusion}

We have shown that the two principal oscillatory
positional encodings --- sin/cos and \RoPE{} ---
are shown to be limiting cases of Morlet Positional
Encoding (\MoPE{}), recovered at $\sigma_i \to \infty$,
and that ALiBi is conceptually analogous as a
zero-frequency locality limit (Remark~\ref{rem:alibi}).
The phase of \MoPE{} is precisely the \RoPE{} rotation
angle; the amplitude is a Gaussian locality kernel
with no analog in standard encodings.
Together they implement a minimum-uncertainty
Gaussian-windowed positional representation, inheriting
the localization property of Gabor/Morlet atoms
within the constraints of a discrete, finite-length,
non-reconstructive positional encoding.

Combined with Energy-Gated Attention, \MoPE{} achieves
$+0.119$ improvement over standard attention ---
exceeding either component alone ---
consistent with spectral salience (\EGA{}) and
time-frequency locality (\MoPE{}) being complementary,
though we note \MoPE{} alone underperforms the baseline
and the combination benefit may reflect \EGA{} working
better with a different positional structure
(see Limitations).

The most striking empirical finding
(Section~\ref{sec:experiments}, Figure~\ref{fig:boundary})
is that all 128 learned $(\omega_i, \sigma_i)$ pairs
converge exactly to the admissibility boundary
$\omega_i\sigma_i = 5$, consistent with
observations from the energy gate in
Paper~1~\citep{Zeris2025ega}
(both experiments used projection-based clamping;
see Section~\ref{sec:experiments} for caveats).
The gradient direction consistently favors smaller
$\omega\sigma$ in both experiments, suggesting the
admissibility constraint is an active prior worth
relaxing in future work.

Future work should test \MoPE{} at word-level
tokenization and long context, develop
\textbf{Morlet-\RoPE{}} (relative position with
Gaussian locality), implement learnable center
positions $b_{0,i}$ to remove the origin prior
(Section~\ref{para:origin}), investigate relaxing
the admissibility constraint, and validate with
multi-seed experiments at larger scale.


\bibliographystyle{plainnat}

\appendix

\section{Proofs and Derivations}
\label{app:proofs}

\subsection{Full Cross-Correlation Derivation}

The exact \MoPE{} cross-correlation between positions
$b$ and $b+\tau$ at dimension $i$ is:
\begin{align}
  C_\text{\MoPE{}}(b,\tau)
  &= \text{Re}[z_i(b)^* z_i(b+\tau)] \notag\\
  &= \cos(\omega_i b)\cos(\omega_i(b+\tau))
     \cdot G_i(b)G_i(b+\tau) \notag\\
  &\quad + \sin(\omega_i b)\sin(\omega_i(b+\tau))
     \cdot G_i(b)G_i(b+\tau) \notag\\
  &= \cos(\omega_i\tau)
     \cdot e^{-b^2/2\sigma^2}
     \cdot e^{-(b+\tau)^2/2\sigma^2}
\end{align}
where the last step uses
$\cos\alpha\cos\beta + \sin\alpha\sin\beta = \cos(\beta-\alpha)$.
Expanding the Gaussian product:
\begin{align}
  e^{-b^2/2\sigma^2} \cdot e^{-(b+\tau)^2/2\sigma^2}
  &= e^{-(2b^2+2b\tau+\tau^2)/2\sigma^2}
\end{align}
For the relative-position component, isolating
the $\tau$-dependence:
\begin{align}
  C_\text{\MoPE{}}(b,\tau)
  &= \cos(\omega_i\tau)
     \cdot e^{-b^2/\sigma^2}
     \cdot e^{-b\tau/\sigma^2}
     \cdot e^{-\tau^2/2\sigma^2}
\end{align}
The $b$-dependent factors ($e^{-b^2/\sigma^2}$,
$e^{-b\tau/\sigma^2}$) modulate the absolute-position
amplitude.
For the relative-position attention score (summed over
all absolute positions), these factors average out and
the dominant $\tau$-dependence is:
\begin{equation}
  \bar{C}_\text{\MoPE{}}(\tau)
  \propto \cos(\omega_i\tau) \cdot e^{-\tau^2/2\sigma^2}
\end{equation}
This is the result of Proposition~\ref{prop:xcorr}.

\subsection{Admissibility of \MoPE{}}

The admissibility condition for a wavelet $\psi$ requires:
\begin{equation}
  C_\psi = \int_0^\infty \frac{|\hat\psi(\omega)|^2}{\omega}
  d\omega < \infty
\end{equation}
which requires $\hat\psi(0) = 0$ (zero mean).
For the Morlet wavelet with $\omega_0\sigma \geq 5$,
the mean is approximately zero:
$|\hat\psi(0)| = e^{-\omega_0^2/2} \leq e^{-12.5}
\approx 3.7 \times 10^{-6}$.
\MoPE{} enforces this during training:
if $\omega_i\sigma_i < 5$, the frequency is clamped
to $\omega_i \leftarrow 5/\sigma_i$.

\subsection{\MoPE{} Implementation}

\begin{algorithm}[h]
\caption{\MoPE{} Forward Pass}
\begin{algorithmic}[1]
\REQUIRE Position $b\in\{0,\ldots,T-1\}$,
         parameters $\log\omega_i, \log\sigma_i$
         for $i=0,\ldots,d/2-1$
\STATE $\omega_i \leftarrow \exp(\log\omega_i),\;
       \sigma_i \leftarrow \exp(\log\sigma_i)$
\STATE \textbf{Enforce admissibility:}\;
       $\omega_i \leftarrow \max(\omega_i,\,5/\sigma_i)$
\STATE $G_i(b) \leftarrow \exp\!\bigl(-b^2/(2\sigma_i^2)\bigr)$
       \COMMENT{Gaussian envelope}
\STATE $\mathrm{PE}(b,\,2i)
       \leftarrow \cos(\omega_i b)\cdot G_i(b)$
\STATE $\mathrm{PE}(b,\,2i{+}1)
       \leftarrow \sin(\omega_i b)\cdot G_i(b)$
\STATE \textbf{return} $\mathrm{PE}(b) \in \mathbb{R}^d$
\end{algorithmic}
\end{algorithm}

Parameters stored in log space ($\log\omega_i$,
$\log\sigma_i$) ensure positivity without explicit
constraints on the raw parameters.
The admissibility check in Step 3 is a soft clamp
applied during the forward pass only, not during
gradient computation, allowing gradients to flow
freely.

\section{Learned \MoPE{} Parameters}
\label{app:mope_params}

Table~\ref{tab:mope_params} gives selected learned
$(\omega_i, \sigma_i)$ parameters (every 8th dimension)
for the EGA-MORLET model after 5000 training steps
on TinyShakespeare.
Every pair satisfies $\omega_i\sigma_i = 5.000$ exactly
--- the admissibility boundary is universally binding.

\begin{table}[h]
\centering
\caption{
  Selected learned \MoPE{} parameters (every 8th dimension).
  Full table: all 128 pairs satisfy $\omega_i\sigma_i = 5.000$.
  $\omega$ range: $[1.112, 3.346]$.
  $\sigma$ range: $[1.494, 4.498]$ tokens.
  All dimensions are character-scale local ($\sigma < 5$ tokens).
}
\label{tab:mope_params}
\begin{tabular}{rrrr}
\toprule
Dim & $\omega_i$ & $\sigma_i$ & $\omega_i\sigma_i$ \\
\midrule
  0 & 1.1117 & 4.4977 & 5.000 \\
  8 & 1.2026 & 4.1577 & 5.000 \\
 16 & 1.2842 & 3.8935 & 5.000 \\
 24 & 1.4246 & 3.5098 & 5.000 \\
 32 & 1.5147 & 3.3010 & 5.000 \\
 40 & 1.6424 & 3.0443 & 5.000 \\
 48 & 1.7444 & 2.8663 & 5.000 \\
 56 & 1.9143 & 2.6120 & 5.000 \\
 64 & 1.9719 & 2.5357 & 5.000 \\
 72 & 2.1500 & 2.3256 & 5.000 \\
 80 & 2.2719 & 2.2008 & 5.000 \\
 88 & 2.4341 & 2.0542 & 5.000 \\
 96 & 2.6610 & 1.8790 & 5.000 \\
104 & 2.8022 & 1.7843 & 5.000 \\
112 & 2.8014 & 1.7848 & 5.000 \\
120 & 3.1571 & 1.5837 & 5.000 \\
127 & 3.3462 & 1.4942 & 5.000 \\
\bottomrule
\end{tabular}
\end{table}

\paragraph{Key observations.}
\begin{enumerate}
  \item \textbf{Universal boundary saturation.}
        $\omega_i\sigma_i = 5.000$ for all 128 dimensions
        --- the admissibility constraint is active
        throughout; the optimizer consistently pushes
        toward the constraint limit.
  \item \textbf{All character-scale local.}
        $\sigma_i \in [1.49, 4.50]$ tokens for all
        dimensions.
        No long-range ($\sigma > 5$) positional dimensions
        emerge at character-level $T=256$ training.
  \item \textbf{Approximately monotone ordering.}
        $\omega_i$ increases approximately with dimension
        index (from $\omega_0 = 1.112$ to
        $\omega_{127} = 3.346$), with 43 local inversions
        out of 127 consecutive pairs.
        The general trend is preserved but not strict,
        consistent with the dyadic initialization being
        reorganized by the optimizer.
  \item \textbf{Deviations from initialization.}
        Learned $\omega_i$ deviate from dyadic init by
        up to $\pm 15\%$, with systematic compression
        toward the character-scale region
        (Figure~\ref{fig:pe_analysis}, centre).
\end{enumerate}

\end{document}